\documentclass{article}
\usepackage[T1]{fontenc}
\usepackage{lmodern}
\usepackage{mathtools}
\usepackage{amsmath}
\usepackage{amssymb}
\usepackage{xcolor}
\usepackage{amsthm}
\usepackage{float}
\usepackage{mathrsfs}
\usepackage{csquotes}
\usepackage{enumerate}
\usepackage{booktabs}
\usepackage{cancel}
\usepackage{tikz}
\usetikzlibrary{arrows.meta,positioning,calc}

\usepackage[ruled,vlined,linesnumbered]{algorithm2e}
\SetKwComment{Comment}{/* }{ */}
\usepackage{booktabs}
\usepackage{pdfpages}
\usepackage[normalem]{ulem}
\usepackage{url}
\usepackage{hyperref}
\usepackage{cleveref}
\usepackage[export]{adjustbox}
\usepackage{geometry}
\usepackage[disable]{todonotes}
\usepackage[english]{babel}

\usepackage{setspace}
\newcommand{\FundingLogos}{%
  \raisebox{0pt}{\includegraphics[height=1.5cm]{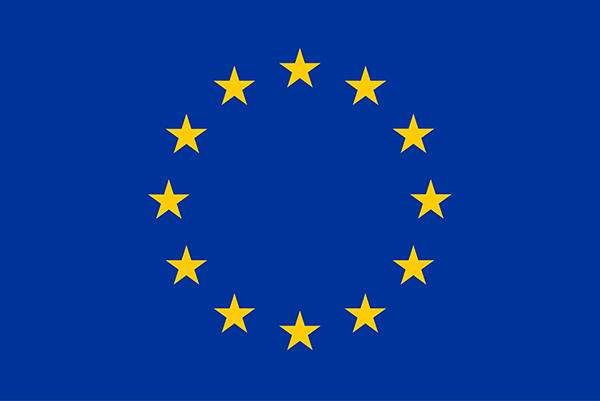}}%
  \hspace{1em}%
  \raisebox{0pt}{\includegraphics[height=1.5cm]{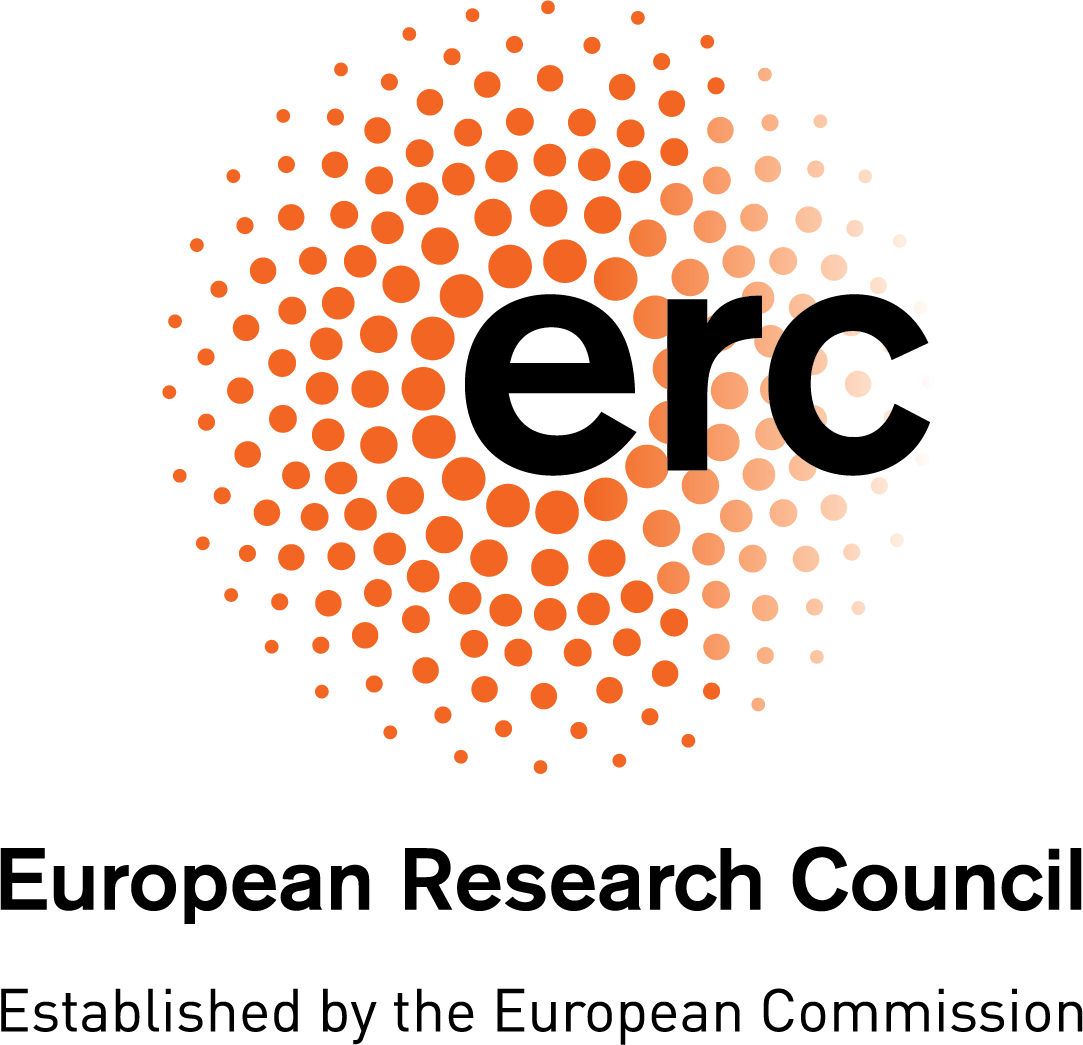}}%
}

\usepackage{tcolorbox}
\tcbset{colback=white,boxrule=0.5pt,arc=2pt,left=4pt,right=4pt,top=4pt,bottom=4pt}

\newcommand{\figref}[1]{Figure~\ref{#1}}

\newcounter{dummy} 
\numberwithin{dummy}{section}

\newtheorem{remark}[dummy]{Remark}

\newtheorem{example}[dummy]{Example}

\newtheorem{cor}[dummy]{Corollary}
\newtheorem{proposition}[dummy]{Proposition}

\usepackage{appendix}

\DeclareMathOperator{\R}{\mathbb R}
\DeclareMathOperator{\N}{\mathbb N}

\DeclareMathOperator{\E}{\mathbb E}
\DeclareMathOperator{\V}{\mathbb V}
\let\P\relax
\DeclareMathOperator{\P}{\mathcal P}
\DeclareMathOperator{\C}{\mathcal C}
\DeclareMathOperator{\F}{\mathcal F}

\DeclareMathOperator{\lin}{lin}
\DeclareMathOperator{\diag}{diag}

\DeclareMathOperator{\Attn}{Attn}
\DeclareMathOperator{\LN}{LN}

\DeclareMathOperator{\MLP}{MLP}

\renewcommand{\d}{\mathop{}\!\mathrm{d}}

\DeclareMathOperator{\tT}{\mathsf{T}}

\DeclareMathOperator{\tr}{tr}
\DeclareMathOperator{\Sym}{Sym}

\usepackage{dsfont}
\DeclareMathOperator{\1}{\mathds{1}}

\DeclareMathOperator{\SM}{SM} 

\newcommand{\hX}{{h_X}}
\newcommand{\hY}{{h_Y}}

\newcommand{\T}{\mathrm{T}}

\title{SympFormer: Accelerated attention blocks via \\Inertial Dynamics on Density Manifolds}
\author{
Viktor Stein%
    \thanks{Institute of Mathematics,
	   Technische Universität Berlin,
	   Stra{\ss}e des 17.\ Juni 136, 
	   10623 Berlin, Germany,
	   {\ttfamily \{stein,steidl\}@math.tu-berlin.de},
      \url{https://tu.berlin/imageanalysis}.
	}\ , 
Wuchen Li%
    \footnote{
    Department of Mathematics,
    University of South Carolina,
    Columbia.
    1523 Greene St, Columbia, SC 29208, USA, {\ttfamily wuchen@mailbox.sc.edu}.
    }\ ,
Gabriele Steidl\footnotemark[1]
}
\date{\today}

\begin{document}
\maketitle

\begin{abstract}
Transformers owe much of their empirical success in natural language processing to the self-attention blocks.
Recent perspectives interpret attention blocks as interacting particle systems, whose mean-field limits correspond to gradient flows of interaction energy functionals on probability density spaces equipped with Wasserstein-$2$-type metrics. 
We extend this viewpoint by introducing accelerated attention blocks derived from inertial Nesterov-type dynamics on density spaces. 
In our proposed architecture, tokens carry both spatial (feature) and velocity variables. The time discretization and the approximation of accelerated density dynamics yield Hamiltonian momentum attention blocks, which constitute the proposed accelerated attention architectures.
In particular, for linear self-attention, we show that the attention blocks approximate a Stein variational gradient flow, using a bilinear kernel, of a potential energy. 
In this setting, we prove that elliptically contoured probability distributions are preserved by the accelerated attention blocks.
We present implementable particle-based algorithms and demonstrate that the proposed accelerated attention blocks converge faster than the classical attention blocks while preserving the number of oracle calls.
\end{abstract}

\section{Introduction}

Transformers \cite{Vaswani2017attention} underpin modern large-scale sequence models and have driven substantial progress in language processing, with large language models, such as ChatGPT and Gemini, becoming part of everyday tools. 
Despite their empirical success, the mathematical principles underlying the stability, scalability, and training efficiency of transformers remain only partially understood \cite{CACP2025,GLPR2024,Geshkovski2025Transformers}. 

The key novelty of the transformer architecture is the attention block, which
is defined by three linear projections of the token features: queries, keys, and values. 
Each token is updated as a weighted aggregation of value vectors with weights determined by the pairwise similarity between queries and keys by a softmax operator.

Recent developments suggest that transformers, and in particular attention blocks, can be interpreted as discrete-time interacting particle dynamics. Each layer update may be viewed as a time- and particle-discretization of non-linear evolution equations in probability density spaces over high-dimensional feature spaces for tokens. Under suitable choices of attention layers, the evolution of tokens resembles gradient descent of an interaction energy functional.
Consequently, the probability density function of the tokens follows the gradient flow of an energy functional in a probability space equipped with a Wasserstein-2-type metric, where the mobility function is induced by the softmax operator applied to query, key, and value vectors.

Gradient flows in Wasserstein spaces have been extensively studied in optimal transport theory, including convergence analysis and numerical schemes; for an overview, see, e.g. \cite{AGS2008,santambrogio2015optimal,O2001}. A classic example is the Fokker-Planck equation \cite{JKO1998}, which is the gradient flow of the free energy in Wasserstein-$2$ space. More recent works extend the Wasserstein spaces and gradient flows with general mobility functions \cite{CLSS2010,DNS2009}, which are also within the scope of dynamical density functional theory \cite{teVrugt02042020}. However, the analytical and computational properties of attention-induced mobilities and their associated interaction energies are still under investigation. Moreover, beyond the architecture's gradient-flow interpolation, training itself corresponds to an optimization procedure over the space of probability densities. These observations motivate a systematic mathematical model of transformers based on variational principles in probability density spaces.

From an optimization perspective, the Nesterov acceleration method~\cite{N1983,Nesterov2004} arises by introducing momentum into gradient descent. This method reduces oscillations and converts slow gradient motion into inertial motion, thereby improving the convergence rate of convex optimization in the discrete-time setting from $O(1/k)$ to $O(1/k^2)$. In continuous time, it corresponds to a second-order dynamical system related to heavy-ball or damped Hamiltonian dynamics~\cite{su2016differential,TM2019}. Such accelerated dynamics have recently been studied on density manifolds endowed with Wasserstein-$2$ type metrics, including accelerated Wasserstein gradient flows and accelerated Stein variational gradient flows \cite{SL2026,SL2026b}. 
Existing connections between transformers and interacting particle systems \cite{L2019} do not yet provide an inertial mechanism, except for the very recent work~\cite{ZPR2026}.
An acceleration principle arises from the geometry of the probability space governing token evolution. If standard transformers correspond to gradient flows on density manifolds, then accelerated transformers may be constructed from inertial dynamics on density manifolds. This observation motivates two natural questions: 
{\em What are ``accelerated attention blocks''
?} {\em Can this accelerated mechanism improve the training procedure, at least numerically?}

This work presents our initial ideas for formulating accelerated transformer architectures based on accelerated dynamics on density manifolds. We apply the fact that transformer layers can be interpreted as particle discretizations of Wasserstein-type gradient flows of interaction energies. We then derive a second-order extension, expressed in terms of momentum and Hamiltonian transformer blocks, in which tokens carry both feature and velocity variables. The resulting dynamics correspond to inertial evolution equations in the probability space. We consider it as an analog of Nesterov's continuous-time acceleration methods for constructing accelerated attention blocks. 

In the literature, several research directions have investigated the mathematical properties of transformers using optimal transport gradient flows. One is Sinkformer \cite{CACP2025}, which establishes a connection between transformers, the Sinkhorn algorithm, and PDEs in the density space.  Motivated by \cite{Peyre2015EntropicWGF, CACP2025,LI2020109449}, a transformer structure has been studied, which results in the regularized Wasserstein proximal operator for sampling algorithms \cite{HanOsherLi2025SparseTransformer,HanOsherLi2025SplittingRWP,TanOsherLi2024NoiseFree}. 
The other direction \cite{Geshkovski2025Transformers, GLPR2023} studies the classical transformer architecture from the perspective of generalized Wasserstein-2 gradient flows. Recently, \cite{CACP2025} investigated transformer PDEs as interacting particle systems. The transformer PDE itself is a gradient flow with respect to the Wasserstein-2 metric with non-linear mobility.
This paper follows the second direction by developing accelerated gradient flows and then designing accelerated attention blocks.
More closely to this paper, we became aware of recent work~\cite{ZPR2026}, which also addresses 
a Nesterov-style accelerated transformer that preserves the same attention and Multilayer perceptron (MLP) oracles.
 However, in contrast to our work, the authors of \cite{ZPR2026} study the ``classical'' Nesterov acceleration in Euclidean space, which 
does not consider the damped Hamiltonian dynamics in density manifolds. We included a numerical comparison with their approach in this paper, and show the advantage in terms of the cross-entropy loss. 

This paper is organized as follows. In Section \ref{sec:Nesterov}, we review Nesterov's acceleration method on Euclidean spaces and in probability density spaces, and geometric integrators for damped Hamiltonian systems.
Then, Section \ref{sec:trans} recalls the transformer architecture
and, in particular, the attention blocks from the perspective of a variational flow on a probability space.
We deal with linear attention layers in Section \ref{sec:lin-att}. We will show that the associated transformer PDE is the Stein gradient flow of a quadratic potential energy. 
This flow preserves elliptically contoured distributions. We provide a Lagrangian spatial discretization that yields a second-order-in-time inertial interacting particle system.
Section \ref{sec:soft-m}  deals with the practically more relevant softmax self-attention. Here, the
transformer PDE is a gradient flow in the softmax-self-attention-mobility Wasserstein space. We derive the
accelerated gradient flow and provide a second-order-in-time inertial interacting particle system. 
Finally, we demonstrate the performance of our accelerated transformer flows by numerical examples in
Section \ref{sec:numerics}.
All proofs and derivations are provided in Appendix \ref {sec:proofs}.
Implementation details are deferred to \Cref{section:Implementation}.

\section{Nesterov's Acceleration Method} \label{sec:Nesterov}
In this section, we briefly review Nesterov accelerated gradient flows in Euclidean space and in the space of probability densities equipped with the Wasserstein-$2$ metric or the Stein-Wasserstein metric. 

Nesterov acceleration~\cite{N1983} is a well-known technique for improving the convergence rate of gradient descent methods. Consider the optimization problem:
\begin{equation}\label{problem}
\min_{x \in \mathbb R^d} F(x),
\end{equation}
where $F\in C^1(\R^d;\R)$ is an objective function. The gradient descent sequence starting at $x_0 \in \R^d$ is
$$x^{(k+1)} = x^{(k)} - \tau \nabla F(x^{(k)}), \qquad x^{(0)} = x_0,$$ 
where $k\in\mathbb{N}$ is the iteration step and $\tau>0$ is a step size. Nesterov's accelerated gradient descent is obtained by choosing an extrapolation coefficient, as in the left part of \eqref{low-c}.
Substituting $t= k \sqrt{\tau}$ and letting $\tau \to 0$, Nesterov’s method can be interpreted through a continuous-time dynamical system~\cite{su2016differential}:
\begin{equation} \label{low-c}
\begin{cases}
x^{(k + 1)} = y^{(k)} - \tau \nabla F(y^{(k)}),  
\\[2mm]
y^{(k + 1)} = x^{(k + 1)} + (1 - \frac{3}{k+3}) (x^{(k + 1)}-x^{(k)}), \\ 
x^{(0)} = x_0, \ y^{(0)} = x_0.
\end{cases}
\qquad
\begin{cases}
\ddot{x}(t) + \dfrac{3}{t} \dot{x}(t) + \nabla F(x(t)) = 0,  
\\[2mm]
x(0) = x_0,\ \dot{x}(0) = 0.
\end{cases}
\end{equation}
For a convex function $F$ and any step size $\tau \in (0,\frac1L)$, 
where $L$ is the Lipschitz constant of $\nabla F$, the discrete scheme exhibits a convergence rate 
$F(x^{(k)}) - F^* \le \mathcal O(\frac{\|x_0-x^*\|}{\tau k^2})$, 
where $x^*$ is a minimizer of $F$ and $F^* = F(x^*)$.
This rate is optimal among all methods having only information about the gradient of $F$ at consecutive iterates \cite{Nesterov2004}. 
For the continuous dynamics, we have similarly
$F(x(t)) - F^* \le \mathcal O(\frac{\|x_0-x^*\|}{t^2})$.
Introducing the momentum variable $p_t$, the second-order ODE \eqref{low-c} can also be rewritten as the first-order system
\begin{equation} \label{first_order}
\dot x_t = p_t, \quad \dot p_t = -\alpha_t p_t - \nabla F(x_t), \qquad x_t|_{t = 0} = x_0, \ p_0 = 0,
\end{equation}
where $\alpha_t \coloneqq \frac3t$. 
Interestingly, it was observed in \cite{TM2019} that with the Hamiltonian function and its derivatives
$$H(x,p) = \frac12 \|p\|^2 + F(x), 
\quad \nabla_p H(x,p) = p, 
\quad \nabla_x H(x,p) = \nabla F(x),$$
the first-order system \eqref{first_order} becomes a damped Hamiltonian flow
\begin{equation} \label{hamil}
\begin{pmatrix}
\dot x_t\\
\dot p_t
\end{pmatrix}
+
\begin{pmatrix}
0\\
\alpha_t p_t
\end{pmatrix}
- 
\left(
\begin{array}{rr}
0&I_d\\
-I_d&0
\end{array}
\right)
\begin{pmatrix}
\nabla_x H(x_t,p_t)\\
\nabla_p H(x_t,p_t)
\end{pmatrix} = 0,
\qquad x_t|_{t = 0} = x_0, \ p_0 = 0.
\end{equation}

In \cite{WL2022}, the damped Hamiltonian flow \eqref{hamil} was directly generalized for accelerating gradient flows on the probability density space.
Let $\Omega =\R^d$ 
and
$$
\mathcal P(\Omega) \coloneqq \{\rho \in C_{>0}^\infty(\Omega): \int_\Omega \rho(x) \, \d x = 1 \}.
$$
For densities $\rho \in \mathcal P(\Omega)$, we denote the tangent space at $\rho$ by  
$$
T_\rho \mathcal P(\Omega) = \{\sigma \in C^\infty(\Omega) : \int_{\Omega} \sigma\d x = 0\},
$$ 
and its cotangent space by 
$T_\rho^*\mathcal P(\Omega) \cong C^\infty(\Omega)/\R$.
For any density function $\rho\in\P(\Omega)$, a metric tensor is an invertible mapping operator $G_\rho: T_\rho \mathcal P(\Omega) \to T_\rho^*\mathcal P(\Omega)$, which defines a metric on $T_\rho (\Omega)$  by
\begin{equation*}
\langle \sigma_1,\sigma_2 \rangle_{\rho} 
\coloneqq
\int_{\Omega}  \sigma_1 G_{\rho} [\sigma_2] \, \d x = 
\int_{\Omega}  \Phi_1 G_{\rho}^{-1} [\Phi_2] \, \d x,
\end{equation*}
where $\sigma_i = G_{\rho}^{-1} \Phi_i$, $i=1,2$ with $\sigma_i\in T_\rho\P(\Omega)$ and $\Phi_i\in T^*_\rho\P(\Omega)$. In this paper, we are mainly interested in the following metrics: 
\begin{align}
\text{Wasserstein}:&\quad (G_\rho^{W})^{-1} \Phi \coloneqq  - \nabla \cdot (\rho \nabla \Phi),\notag\\
\text{Stein}:& \quad (G_\rho^{S})^{-1} \Phi 
\coloneqq  - \nabla \cdot \left(\rho \int_\Omega k(\cdot,y) \, \nabla \Phi(y) \, \rho(y) \, \d y\right), \label{steingf}
\end{align}
 where $k \colon \Omega \times \Omega \time \Omega \to \mathbb R$ is a symmetric, integrally strictly positive definite kernel with $\int_\Omega k(x,x) \d x < \infty$.
Then, we study the minimization problem of a functional $\mathcal{F}\colon\P(\Omega)\to\R$ in the probability density space: 
\begin{equation*}
\min_{\rho \in \mathcal P(\Omega)} \mathcal F(\rho).\label{eq:F}
\end{equation*}
The gradient flow of energy functional $\mathcal{F}(\rho)$ in the probability density space  $(\P(\Omega), G)$ satisfies 
\begin{equation} \label{gradd}
\partial_t \rho_t =  - G_{\rho_t}^{-1} \Big[\frac{\delta \mathcal F(\rho_t)}{\delta \rho_t} \Big].
\end{equation}
Here,
$
\frac{\delta \mathcal F(\rho_t)}{\delta \rho_t} \in T^*_\rho \mathcal P(\Omega)
$ 
denotes the $L^2$ first variation of functional $\mathcal{F}(\rho)$ at the density $\rho$ defined by
 $$\int_\Omega  \frac{\delta \mathcal F(\rho_t)}{\delta \rho_t}   \varphi \d x
 = \lim_{\varepsilon \to 0} \frac{1}{\varepsilon} \left( \mathcal F(\rho + \epsilon \varphi) - \mathcal{F}(\rho) \right),$$ 
for all $\varphi \in C^\infty(\Omega)$ with mass zero for which 
$\rho + \epsilon \varphi \in \mathcal P(\Omega)$.
Then the acceleration dynamics \eqref{hamil} in $(\P(\Omega), G)$, suggested in \cite{WL2022}, satisfies 
\begin{equation} \label{sys}
\partial_t
\begin{pmatrix}
\rho_t\\
\Phi_t
\end{pmatrix}
+
\begin{pmatrix}
0\\
\alpha_t \Phi_t
\end{pmatrix}
- 
\begin{pmatrix}
0&I\\
-I&0
\end{pmatrix}
\begin{pmatrix}
\frac{\delta}{\delta \rho_t}  \mathcal{H}(\rho_t,\Phi_t)\\[0.5ex]
\frac{\delta}{\delta \Phi_t}  \mathcal{H}(\rho_t,\Phi_t)
\end{pmatrix} = 0,
\qquad \rho_t|_{t = 0} = \rho_0, \ \Phi_0 = 0,
\end{equation}
where $\rho_0\in\P(\Omega)$ is a given initial density function, $I$ denotes the identity operator on 
$T^*_\rho \mathcal P(\Omega)$,
and the Hamiltonian functional $\mathcal{H}\colon P(\Omega)\times T_\rho^*\P(\Omega)\to\R$ satisfies 
\begin{equation} \label{hamilt}
 \mathcal{H}(\rho_t,\Phi_t) \coloneqq \frac12 \int_\Omega \Phi_t G_{\rho_t}^{-1} [\Phi_t] \, \d x + \mathcal F(\rho_t).
\end{equation}
Indeed, this accelerated gradient method
relies on two ingredients, namely the functional $\mathcal F$ and the metric tensor $G_\rho$. The system \eqref{sys} can be rewritten as
\begin{equation}\label{sys_1}
\begin{cases}
 \partial_t \rho_t - G^{-1}_{\rho_t}   [\Phi_t] = 0, & \rho_t|_{t = 0}=\rho_0, \\[1ex]
 \partial_t \Phi_t + \alpha_t \Phi_t + \frac12 \frac{\delta}{\delta \rho_t}
 \left(\int_\Omega \Phi_t G^{-1}_{\rho_t}[\Phi_t]  \d x \right)  + \frac{\delta \mathcal F}{\delta \rho_t} = 0, & \Phi_0 = 0.
 \end{cases}
\end{equation} 
For example, the accelerated gradient flow \eqref{sys_1} in these density manifolds becomes 
\begin{align}
\text{Wasserstein:} \quad &
\begin{cases}
\partial_t \rho_t + \nabla \cdot \left(\rho_t \nabla  \Phi_t  \right) = 0, & \rho_t|_{t = 0} = \rho_0, \\[1ex]
\partial_t \Phi_t + \alpha_t  \Phi_t + \frac12 \| \nabla \Phi_t\|_2^2 + \frac{\delta \mathcal F(\rho_t)}{\delta \rho_t} = 0, & \Phi_0 = 0.
\end{cases}
\\
\text{Stein:} \quad &
\begin{cases}
\partial_t \rho_t + \nabla \cdot \Big( \rho_t(x) \int_{\Omega} k(x,y) \rho_t(y) \nabla \Phi_t(y) \, \d y \Big) = 0, & \rho_t|_{t = 0} = \rho_0, \\[1ex]
\partial_t \Phi_t + \alpha_t  \Phi_t + \int_\Omega\nabla \Phi_t(x)^\T \nabla \Phi_t(y) \,  k(x,y) \, \rho_t(x) \, \d x + \frac{\delta \mathcal F(\rho_t)}{\delta \rho_t} = 0, \; \; &\Phi_0 = 0.
\end{cases}
\end{align}
This formalism is only rigorous on compact manifolds; we treat the case of $\Omega = \R^d$ only formally.
A fruitful avenue for future research is to model layer normalization as a projection onto the sphere, thereby placing the dynamics on a compact manifold.

We will discretize certain accelerated gradient flows by approximating the density $\rho_t$ by the empirical measure 
$\frac{1}{N} \sum_{j = 1}^{N} \delta_{X_i(t)}$ and setting $Y_i(t) \coloneqq \nabla \Phi_t(X_i(t))$, yielding a finite-dimensional linearly damped Hamiltonian system in Euclidean space \eqref{hamil}.
Later, we will leverage that \eqref{hamil} can be rewritten as the \textit{undamped} Hamiltonian system
\begin{equation} \label{eq:time_dependent_Hamiltonian_system}
    \begin{cases}
        \dot{x}(t)
        = \partial_P \tilde{H}\big(t, x(t), P(t)\big), & x(0) = x_0, \\
        \dot{P}(t)
        = - \partial_q \tilde{H}\big(t, x(t), P(t)\big), & P(0) = 0,
    \end{cases}
\end{equation}
with the \textit{time-dependent} Hamiltonian
\begin{equation*}
    \tilde{H}
    (t, q, P) \coloneqq e^{\eta(t)} H(q, e^{- \eta(t)} P), \qquad t \ge t_0,
\end{equation*}
using the change of variables
\begin{equation} \label{eq:change_of_variables}
    P = e^{\eta(t)} p,   
\end{equation}
see, e.g., \cite[App.~B]{FJV2021} and also \cite{WWJ2016}.

\paragraph{Time Discretizations of Linearly Damped Hamiltonian Systems} \label{subsec:time_discretizations}

Let $\eta \colon [t_0, \infty) \to \R$, $t \mapsto \int_{t_0}^{t} \alpha(s) \d{s}$ be a well-defined antiderivative for $\alpha$, for some fixed $t_0 \ge 0$.
For example, if $\alpha(t) \equiv \mu > 0$, then we can pick $t_0 = 0$ and then $\eta(t) = \mu t$, whereas if $\alpha(t) = \frac{3}{t}$, then we will choose $t_0 > 0$ and obtain $\eta(t) = 3 \ln\left(t / t_0 \right)$.

Further, let $h_k > 0$ denote the step sizes.
We use the time grid $t_{k + 1} \coloneqq t_k + h_k$ for $k \in \N$ for some $t_0 \ge 0$ and denote by $\sigma_k \coloneqq \exp\left(-\int_{t_k}^{t_{k + 1}} \alpha(s) \d{s}\right)$ the discrete damping coefficient for $k \in \N$, which mimics the continuous-time dissipation \cite{AA2024}.
For common damping schedules $\alpha$, the damping coefficient is available in closed form.

\paragraph{Explicit Euler method}
    The most simple time discretization of \eqref{hamil} is a \textit{plain explicit forward Euler} discretization with step size $\tau > 0$, which does not take any geometry into account.
    Its $k$-th iteration is
    \begin{equation} \label{eq:plain_explicit_Euler}
        \begin{cases}
            x^{(k + 1)}
            = x^{(k)} + h_k \nabla_p H(x^{(k)}, p^{(k)}), \\
            p^{(k + 1)}
            = \alpha_k p^{(k)} - h_k \nabla_x H(x^{(k)},p^{(k)}),
        \end{cases}
    \end{equation}
    where $\alpha_k = \frac{k - 1}{k + 2}$ or $\alpha_k$ is constant.
    This Nesterov-style discretization is not conformally symplectic \cite{FJV2021,FSRV2020}.

    \paragraph{Conformally symplectic Euler method}
    A geometric approach \cite{GH2025,BJW2018} is to split the vector field $(\partial_p H, - \alpha p - \partial_q H)$ driving the evolution \eqref{hamil} into its conservative part $(\partial_p H, - \partial_q H)$ and the dissipative part $(0, - \alpha p)$.
    We apply a Lie-Trotter split: integrate the dissipative part exactly and use a symplectic integrator for the conservative part.
    This is known as the \textit{conformally symplectic Euler} or \textit{\enquote{kick then damp} scheme}, which reads
    \begin{equation} \label{eq:presymp_Euler} 
        \begin{cases}
            p^{(k + 1)}
            = \sigma_k \left( p^{(k)} - h_k \nabla_x H(x^{(k)}, p^{(k)}) \right), \\
            x^{(k + 1)}
            = x^{(k)} + h_k \nabla_p H(x^{(k)}, p^{(k + 1)}).
        \end{cases}
    \end{equation}
    
\paragraph{Exponential Euler method}
    For a simple Hamiltonian, whose kinetic energy term is just $\frac{1}{2} \| p \|^2$, the dynamics \eqref{hamil} can be written as $\dot{p} + \alpha p - \nabla_x H(x, p) = 0$.
    Hence, we can use the \textit{exponential Euler method} \cite{HOS209} applied to this simplified evolution to approximately discretize \eqref{hamil} for general Hamiltonians.
    The exponential Euler update for fixed step size $h_k = h > 0$ reads as follows:
    \begin{equation} \label{eq:Presympl_exp_euler}
        \begin{cases}
        p^{(k + 1)}
        = \sigma_k p^{(k)}
        - h \int_{0}^{h} \exp\left( - \int_{s}^{h} \alpha(z) \d{z}\right) \d{s} \nabla_x H(x^{(k)}, p^{(k)}), \\
        x^{(k + 1)}
        = x^{(k)} + h \nabla_p H(x^{(k)}, p^{(k)}).
        \end{cases}
    \end{equation}
    For $\alpha(s) = \frac{r}{s}$, the weight in front of $\nabla_x H$ in the first update equation becomes $\frac{h}{r + 1}$ and for $\alpha(s) = m$, it becomes $\frac{1 - e^{- m h}}{m}$.
    For the log-linear damping $\alpha(s) = \frac{r}{s} + m$, it can be expressed in terms of the confluent hypergeometric function ${}_1 F_1$.\\[3mm]

For the accelerated transformer architecture presented in \Cref{sec:lin-att,sec:soft-m}, we need time discretizations that require a small number of function evaluations (also called oracle calls).
First-order time-discretizations evaluate the function only once per iteration.
Multi-step methods can achieve higher orders of accuracy than first-order methods by reusing function values from previous iterations, thereby preserving the oracle budget.  
In contrast, higher-order single-step methods like Runge-Kutta-4 need four oracle calls per step.

\paragraph{Adams-Bashforth (AB-2) method}
The two-step Adams-Bashforth (AB-2) \cite{B1883} is an explicit linear multistep method of \textit{order two} -- the best order achievable by an explicit two-step multistep method \cite[Chp.~3]{HNW93}.
We can write \eqref{hamil} abstractly as an ODE of the form $y'(t) = f(t, y(t))$, where $y = (x, p)$. The AB-2 method then is
\begin{equation} \label{eq:AB2}
    y^{(k + 1)} = y^{(k)} + h_k \left( \frac{2 h_{k - 1} + h_k}{2 h_{k - 1}} f(t^{(k)}, y^{(k)}) - \frac{h_k}{2 h_{k - 1}} f(t^{(k - 1)},y^{(k - 1)})\right).
\end{equation}
Here $h_{k}>0$ is a stepsize.  To obtain the value $y^{(1)}$ from the given initial value $y^{(0)}$, we perform an explicit Euler step: $y^{(1)} = y^{(0)} + h_0 f(t_0, y^{(0)})$.

\section{Transformers} \label{sec:trans}
In this section, we briefly review transformers from the viewpoint of PDEs on the space of probability densities; for details, see, e.g. \cite{CACP2025,Geshkovski2025Transformers}. 

Transformers modify token configurations $(X_1,\ldots, X_N) \in (\mathbb R^d)^N$ using two essential components, namely self-attention layers, which combine information across tokens, and MLP layers, which perform independent non-linear transforms on each token. Let
$Q, K \in \R^{m\times d}$ and $V \in \R^{d\times d}$, $m\leq d$,  be query, key, and value matrices, which usually differ in each layer of the transformer, 
and set $A \coloneqq K^\T Q \in \R^{d\times d}$. 
We assume that $A$ is a symmetric matrix.
In this paper, we are interested in functions $\kappa = \kappa_A: \R^d \times \R^d \to \R$ in one of the following forms:
\begin{equation} \label{eq:self_attn_kernel}
\kappa(X_i,X_j) = 
\left\{
\begin{array}{ll}
   \frac{\exp\left((X_j)^\T A X_i \right)}{\sum_{\ell=1}^N \exp\left((X_{\ell})^\T A X_i\right)}
   =
   \text{softmax}_j\big((X_\ell^{\tT} A X_i)_{\ell = 1}^{N}\big) &\quad  \text{softmax self-attention},\\[3ex]
 (X_j)^\T A X_i &\quad \text{linear self-attention}.
\end{array}
\right.
\end{equation}
Recall that the softmax function is defined by 
$
\text{softmax} (w) \coloneqq \big(\exp(w_i) /\sum_{j=1}^d \exp(w_j) \big)_{i=1}^d
$, for $w\in\R^d$. Then, one layer of a transformer consists of the sublayers: 
\\[1ex]
\textbf{Self-attention layer}:
\begin{equation} \label{eq:self-attention}
X^{(k + 1)}_i = X^{(k)}_i +h \sum_{j=1}^N \kappa(X^{(k)}_i,X^{(k)}_j) VX^{(k)}_j.
\end{equation}
\\
\textbf{MLP layer}:
\begin{equation} \label{eq:MLP}
    X^{(k + 1)}_i= X^{(k)}_i +  hW \sigma (\tilde W X^{(k)}_i + b),     
\end{equation}
where we denote the input tokens as $(X^0_1,\cdots, X^0_N)$, $h>0$ is a step size, $k\in\mathbb{N}$ is the depth of the layer, $W,\tilde W\in\R^{d\times d}$ are the weight matrices, and the bias vector $b\in\R^d$ differs in general in each layer, 
and $\sigma\colon \R\to\R$ is a non-linear function applied componentwise. 
This description assumes that the matrices $Q$, $K$, and $V$ are the same in each layer (\enquote{weight sharing}), which is the case for some transformer architectures, such as ALBERT \cite{Lan2020}.

We next present an ODE system only for the attention layer described above. As the step size $h\to 0$, we obtain the following ODE system. Given input tokens 
$\mathbf{X}(0) = (X_1(0), \dots, X_N(0)) \in (\R^d)^N$, their evolution through an attention layer can be modeled as a dynamical system
\begin{equation}\label{eq:attention}
\dot{X}_i(t) = \Gamma_{X(t)}(X_i(t)), 
\qquad 1 \le i \le N,
\end{equation}
where $\Gamma_{X}\colon \R^d\to\R^{d}$ denotes the attention vector field determined by the query, key, and value matrices $Q$, $K$, $V$ for the current total tokens $X(t)$. 
In the case of softmax self-attention, the vector field takes the form: 
\begin{equation*}
 \Gamma_{X}(x)\coloneqq 
\sum_{j=1}^N
\frac{\exp\big(Qx \cdot KX_j\big)}
{\sum_{\ell=1}^N \exp\big(Qx \cdot KX_\ell\big)}
\, VX_j,\qquad  \textrm{for }x\in\R^d.
\end{equation*}
Each token in \eqref{eq:attention} evolves as a similarity-weighted average of the value vectors, which follows a non-linear interacting particle system.

By \cite{CACP2025,Geshkovski2025Transformers}, see also  \cite{BKKRW2025,GLPR2024,SABP2022,ZZCD2022}, the mean-field limit of the interacting particle system \eqref{eq:attention} with only self-attention blocks can be expressed for $\Omega = \R^d$ 
as
\begin{equation} \label{eq:transformer_PDE}
    \partial_t \rho_t
    = - \nabla \cdot (\rho_t \Gamma_{\rho_t}), \qquad t > 0,
\end{equation}
where the velocity field is a probability density dependent vector function satisfying:
\begin{equation*}
    \Gamma_{\rho} \colon \R^d \to \R^d, \qquad
    x \mapsto \int_{} V y \, \kappa_{\rho}(x, y) \rho(y)  \d y.
\end{equation*}
In this paper, we are interested in the following two attention layers: 
\begin{equation*}
    \kappa_{\rho}(x, y)
    \coloneqq \begin{cases} 
    \frac{\exp(y^\T A x)}{\int_{} \exp(z^\T A x) \rho(z) \d z} & \text{softmax self-attention},
    \\[2ex]
    y^\T A x, & \text{linear self-attention}.
    \end{cases}
\end{equation*}
As shown for accelerating the gradient flow \eqref{gradd} in the previous section,
we propose to accelerate the flow \eqref{eq:transformer_PDE}.
We start with the simple linear self-attention in Section \ref{sec:lin-att} and then consider 
the softmax self-attention in Section \ref{sec:soft-m}. 
\section{Accelerated Transformers with Linear Attention}
\label{sec:lin-att}
In this section, for the linear self-attention, we show that the associated transformer PDE is the Stein variational gradient flow of a quadratic potential energy. We prove that this flow preserves elliptically contoured distributions and provide a spatial discretization that yields a second-order-in-time inertial interacting particle system. 
In this section, we assume that $A$ is symmetric and positive definite.

\subsection{Accelerated Flow} 
For the linear attention, the transformer PDE \eqref{eq:transformer_PDE} reads as
\begin{equation} \label{eq:transformer_PDE_lin}
    \partial_t \rho_t
    = - \nabla \cdot \Big( \rho_t(x) \int y^\T A x \, Vy \rho_t (y) \d y \Big).
\end{equation}
Assume that $V$ is symmetric.
For the potential energy $\F^{\lin} \colon \P(\R^d) \to \R \cup \{ \pm \infty \}$  given by
\begin{equation*}
    \F^{\lin}(\rho)
    \coloneqq - \frac{1}{2} \int_{} y^{\T} \, V y \, \rho(y) \d y,
\end{equation*}
we have that 
$$
\Big[ \frac{\delta \F^{\lin}(\rho)}{\delta \rho}\Big](y) = - \frac12 y^\T V y, \qquad \nabla_y \Big[ \frac{\delta \F^{\lin}(\rho)}{\delta \rho}\Big] = - Vy.
$$
Thus, defining the inverse metric tensor by
$$
(G^{\lin}_{\rho})^{-1}[\Phi] 
   \coloneqq - \nabla \cdot \left(\rho(x) \, \int_{} y^{\T} A x\,  \nabla \Phi(y) \, \rho(y) \d y\right),
$$
we conclude that
\begin{equation*}
    \partial_t \rho_t = - (G_{\rho_t}^{\lin})^{-1}\left[ \frac{\delta \F^{\lin}(\rho_t)}{\delta \rho_t} \right].
\end{equation*}
This transformer PDE is exactly the Stein variational gradient flow \eqref{steingf}
with respect to the energy functional $\F^{\lin}$ with kernel $k(x, y) \coloneqq y^{\T} A x$.



For the linear self-attention, the accelerated gradient flow in \eqref{sys_1} can be specified as in 
the following proposition. The proof is given in  Appendix \ref{sec:proofs}.

\begin{proposition}[Accelerated linear attention PDE] \label{lemma:acc_linear}
    The accelerated linear self-attention  transformer flow in \eqref{sys_1} satisfies
    \begin{equation} \label{eq:acc_lin_transformer_PDE}
        \begin{cases}
            \partial_t \rho_t(x)
            + \nabla \cdot \left( \rho_t(x) \int_{} y^{\T} A x \nabla \Phi_t(y) \rho_t(y) \d{y} \right) = 0, & \rho_t|_{t = 0} = \rho_0 \\[1ex]
            \partial_t \Phi_t(x)
            + \alpha_t \Phi_t(x)
            + \int_{} \langle \nabla \Phi_t(x), \nabla \Phi_t(y) \rangle y^{\T} A x \rho_t(y) \d{y}
            - \frac{1}{2} x^{\T} V x=0, & \Phi_0 = 0.
        \end{cases}
    \end{equation}
    Here, $(\alpha_t)_{t > 0}$ are non-negative damping parameters.
\end{proposition}

\subsection{Preservation of Elliptically Contoured Distributions}

We show that if the initial density $\rho_0$ is elliptically contoured, e.g., a Gaussian distribution, 
then the solution $\rho_t$ of \eqref{eq:acc_lin_transformer_PDE} will remain so for all times.

A probability density $\rho$ on $\R^d$ is called \emph{elliptically contoured} \cite[Sec.~5]{S2002} \cite[Sec.~3]{FJS2003}. We write $\rho = E(m, \Sigma, g)$, if $\rho$ can be expressed as
\begin{equation*}
    \det(\Sigma)^{-\frac{1}{2}} g\left( (\cdot - m)^{\T} \Sigma^{-1} (\cdot - m)\right),
\end{equation*}
for some    
$m \in \R^d$, $\Sigma \in \Sym_+(\R; d)$ and some smooth, integrable function $g \colon [0, \infty) \to (0, \infty)$ fulfilling
\begin{itemize}
    \item[i)] normalization:
    $\int_{0}^{\infty} r^{\frac{d}{2} - 1} g(r) \d{r} = \frac{\Gamma\left(\frac{d}{2}\right)}{\pi^{\frac{d}{2}}} $;
    \item[ii)] finite expectation:
    $\int_{0}^{\infty} r^{\frac{d - 1}{2}} g(r) \d{r} < \infty$;
    \item [iii)] finite variance:
    $\int_{0}^{\infty} r^{\frac{d}{2}} g(r) \d{r} < \infty$.
\end{itemize}

We study elliptically contoured distributions because this family is invariant under the pushforward mapping function by affine linear functions:
given the affine linear map $f(x) \coloneqq A x + b$ for a matrix $A \in \R^{d \times d}$ and a vector $b \in \R^d$, we have that $\rho = E(m, \Sigma, g)$ implies that the pushforward density of $\rho$ by $f$ is elliptically contoured, namely $E(A m + b, A \Sigma A^{\T}, g)$,
see \cite[p.~279]{FJS2003}.

\begin{example}[Elliptically contoured distributions]
    For $g(x) = (2 \pi)^{-\frac{d}{2}}e^{-\frac{1}{2} x}$, we obtain \textit{Gaussian distributions}, and for $g(x) = \tfrac{\Gamma\left(\frac{v + d}{2}\right)}{\Gamma\left(\frac{v}{2}\right)} (v \pi)^{-\frac{d}{2}} \left(1 + \frac{1}{v} x\right)^{-\frac{v + d}{2}}$, we get the multivariate \textit{Student-t distributions} with $v > 2$ degrees of freedom.
    There are also other less-known examples, such as logistic distributions, even power-exponential distributions, and generalized hyperbolic distributions, see \cite[Sec.~6]{S2002}.
\end{example}
\vspace{0.5cm}

The expectation value and the variance of an elliptic distribution are, respectively
$$
\E_{X \sim E(m, \Sigma, g)}[X] = m, \qquad  \V_{X \sim E(m, \Sigma, g)}[X] = \kappa_g \Sigma,\quad \text{where } \kappa_g \coloneqq \frac{1}{d} \int_{} \| y \|^2 g(\| y \|^2) \d{y}.
$$
For Gaussian distributions, we have $\kappa_g = 1$ and for the Student-$t$ distribution with $v$ degrees of freedom $\kappa_g = \frac{v}{v - 2}$.

 By the following proposition, whose proof is given in Appendix~\ref{sec:proofs}, the accelerated linear transformer PDE \eqref{eq:acc_lin_transformer_PDE} preserves elliptically contoured distributions.

\begin{proposition}[Preservation of elliptically contoured distributions] \label{lemma:preservation_elliptic}
    Let 
    $(\rho_t, \Phi_t)_{t \ge 0}$ satisfy \eqref{eq:acc_lin_transformer_PDE} and $\rho_0 = E(m, \Sigma, g)$.
    Then it holds
    \begin{equation*}
        \rho_t = E(m_t, \Sigma_t, g)
        \qquad \text{and} \qquad
        \Phi_t(x) = \frac{1}{2} x^{\T} P_t x,
    \end{equation*}
   with $m_t$ and $\Sigma_t, P_t \in \Sym(\R; d)$ determined by 
    \begin{equation}\label{determ}
       \begin{cases}\dot{m}_t
            = C_t m_t, & m_t|_{t = 0} = m_0, \\
            \dot{\Sigma}_t
            = \Sigma_t C_t^{\T} + C_t \Sigma_t, & \Sigma_t|_{t = 0} = \Sigma_0,
        \end{cases}
    \end{equation}
    where 
    $C_t \coloneqq P_t (\kappa_g \Sigma_t + m_t m_t^{\T}) A$ and
    \begin{equation} \label{eq:C_t_and_P_t}
        \begin{cases} \dot{P}_t
        = - \alpha_t P_t - (C_t^{\T} P_t + P_t C_t)
        + V, & t > 0, \\
        P_0 = 0.
        \end{cases}
    \end{equation}
\end{proposition}

If the initial density $\rho_0$ has zero mean, then we obtain the following more simple evolution of $\rho_t$.
\begin{cor}[Preservation of centered elliptically contoured distributions]
    Let $(\rho_t, \Phi_t)_{t \ge 0}$ satisfy \eqref{eq:acc_lin_transformer_PDE} and $\rho_0 = E(0, \Sigma, g)$. 
    Then,
    \begin{equation*}
        \rho_t = E(0, \Sigma_t, g)
        \qquad \text{and} \qquad
        \Phi_t(x) = \frac{1}{2} x^{\T} P_t x,
    \end{equation*}
    with $\Sigma_t, P_t \in \Sym(\R; d)$ determined by 
    \begin{equation*}
        \dot{\Sigma}_t
        = \kappa_g \left( \Sigma_t A \Sigma_t P_t + P_t \Sigma_t A \Sigma_t\right),
    \end{equation*}
    and
    \begin{equation} \label{eq:C_t_and_P_t_zero_mean}
        \begin{cases}
            \dot{P}_t
            = - \alpha_t P_t - \kappa_g( A \Sigma_t P_t^2 + P_t^2 \Sigma_t A)
            + V, & t > 0, \\
            P_0 = 0.        \end{cases}
    \end{equation}
\end{cor}

\subsection{Accelerated Linear Attention Dynamics}

We now discretize \eqref{eq:acc_lin_transformer_PDE} in space. For $t > 0$, we approximate the density $\rho_t$ by the empirical measure 
$\frac{1}{N} \sum_{j = 1}^{N} \delta_{X_i(t)}$.
In Appendix \ref{sec:proofs}, we provide the following discretization result.

\begin{proposition}\label{prop:add}
    The deterministic interacting particle system associated to \eqref{eq:acc_lin_transformer_PDE} is
    \begin{equation*} \label{eq:acc_linear_particles}
        \begin{cases}
            \dot{X}_i(t)
            = \frac{1}{N} \sum_{j = 1}^{N} Y_j(t) X_j(t)^{\T} A X_i(t), & X_i(0) = X_{i, 0}, \\ 
            \dot{Y}_i(t)
            = - \alpha_t Y_i(t) - \frac{1}{N} \sum_{j = 1}^{N} \langle Y_i(t), Y_j(t) \rangle A X_j(t) + V X_i(t), & Y_i(0) = 0,
        \end{cases}
    \end{equation*}
    where $Y_i(t) = \nabla \Phi_t(X_i(t))$.
    In the matrix form, the above system satisfies
    \begin{equation} \label{eq:acc_linear_matrix}
        \begin{cases}
            \dot{\mathbf X}(t)
            = \frac{1}{N} \mathbf X(t) A \mathbf X(t)^{\T} \mathbf Y(t), & \mathbf{X}(0) = X_0, \\
            \dot{\mathbf Y}(t)
            = - \alpha(t) \mathbf Y(t)
            - \frac{1}{N} \mathbf Y(t) \mathbf Y(t)^{\T} \mathbf X(t) A
            + \mathbf X(t) V, & \mathbf{Y}(0) = \mathbf{0}.
        \end{cases}
    \end{equation}
\end{proposition}

The explicit Euler discretization \eqref{eq:plain_explicit_Euler} with step sizes $(h_k)_{k \in \N}$ of \eqref{eq:acc_linear_matrix} is
\begin{equation*}
    \begin{cases}
    \mathbf{X}^{(k + 1)}
    = \mathbf{X}^{(k)} + \frac{h_k}{N} \mathbf{X}^{(k)} A (\mathbf{X}^{(k)})^{\tT} \mathbf{Y}^{(k)}, & \mathbf{X}^{(0)} = \mathbf{X}_0, \\
    \mathbf{Y}^{(k + 1)}
    = \alpha_k \mathbf{Y}^{(k)} - \frac{h_k}{N} \mathbf{Y}^{(k)} \left( \mathbf{Y}^{(k)}\right)^{\tT} \mathbf{X}^{(k)} A + \mathbf{X}^{(k)} V, & \mathbf{Y}^{(0)} = \mathbf{0}.
    \end{cases}
\end{equation*}

\section{Accelerated Transformers with Softmax attention} \label{sec:soft-m}
In this section, we present the main result of this paper.
For the softmax self-attention, we briefly review that the transformer PDE is a gradient flow in a {\em softmax self-attention mobility Wasserstein space} $(\P, G_\rho^{\mathrm{SM}})$ \cite{Geshkovski2025Transformers}.
We then derive the
accelerated gradient flow in $(\P, G_\rho^{\mathrm{SM}})$, for which we provide first- and second-order-in-time inertial interacting particle systems for the softmax self-attention layer. 
\subsection{Accelerated Flow} 
For the softmax self-attention, the transformer PDE \eqref{eq:transformer_PDE} reads as
$$
\partial_t \rho_t = - \nabla \cdot \Big( \rho_t(x) \int Vy \frac{\exp(y^\T Ax)}{\int \exp(z^\T Ax) \rho_t(z) \d z} \rho_t(y) \d y
\Big).
$$
For the energy $\F^{\SM} \colon \P(\R^d) \to \R \cup \{ \pm \infty \}$  given by
\begin{equation} \label{eq:softmaX_interaction_energy}
        \F^{\text{SM}}(\rho)
        \coloneqq 
        - \frac{1}{2} \int_{} \int_{} \exp( y^\T A x) \rho(x) \rho(y) \d x \d y,
    \end{equation}
we have that 
\begin{align}
\Big[ \frac{\delta \F^{\SM}(\rho)}{\delta \rho}\Big] 
&= 
- \int_{} \exp\left( y^{\T} A x\right) \rho(y) \d y, \label{helper_a}
\\
\nabla_x \Big[ \frac{\delta \F^{\SM}(\rho)}{\delta \rho}\Big] &= - \int \exp\left( y^{\T} A x\right) \, A y \,\rho(y) \d y.
\end{align}
Thus, defining the inverse metric tensor by
\begin{equation} \label{eq}
        (G_{\rho}^{\SM})^{-1}[\Phi]
        \coloneqq   
        -  \nabla \cdot \Big(\rho \, \frac{V A^{-1} \nabla_x \Phi}{\int \exp \left( z^\T A x\right) \rho(z) \d z} \Big).
    \end{equation}
We observe that the transformer PDE satisfies
\begin{equation*}
    \partial_t \rho_t = - (G_{\rho_t}^{\SM})^{-1}\left[ \frac{\delta \F^{\SM}(\rho_t)}{\delta \rho_t} \right].
\end{equation*}
Assume that $$B \coloneqq VA^{-1},$$ is symmetric and positive definite. Then the transformer PDE resembles the Wasserstein-2 type gradient flow with a nonlinear mobility matrix function
    $\chi: \mathcal P(\Omega) \to \mathcal \C^{\infty}(\Omega)^{d \times d}$:
    $$\partial_t \rho_t =  \nabla \cdot \Big(\chi(\rho_t) \nabla \, \Big[\frac{\delta \F^{\text{SM}}(\rho_t)}{\delta \rho_t} \Big]\Big), \qquad 
        \chi(\rho)(x)
        \coloneqq  \frac{\rho(x)}{\int_{} \exp\left(z^\T A x\right) \rho(z) \d z} V A^{-1}.
    $$
Furthermore, we introduce the notation $\| v \|_B^2 \coloneqq \langle v, B v \rangle$, $v \in \R^d$.
The following proposition specifies the accelerated gradient flow in \eqref{sys_1}
for the softmax self-attention layer. The proof is provided in \ref{sec:proofs}.

\begin{proposition}[Accelerated softmax attention PDE] \label{lemma:acc_softmax}
      The accelerated softmax attention flow
    in \eqref{sys_1} satisfies
    \begin{equation} \label{eq:acc_softmax}
        \begin{cases}
            \partial_t \rho_t(x)
            + \nabla \cdot \left(\rho_t(x) \displaystyle\frac{V A^{-1} \nabla \Phi_t(x)}{\int_{} \exp\left(z^\T A x\right) \rho_t(z) \d z}\right)=0, & \rho_t|_{t = 0} = \rho_0,
            \\[2em]
            \partial_t \Phi_t(x)
            +\alpha_t \Phi_t(x)
            + \displaystyle\frac{1}{2} \frac{\| \nabla \Phi_t(x) \|_B^2}{\int_{} \exp\left(z^\T A x\right) \rho_t(z) \d z} 
            \\[2ex]
             \quad - \displaystyle\frac{1}{2} \int \exp\left(y^\T A x\right) \left( \frac{\| \nabla \Phi_t(y) \|_B^2}{\left(\int \exp\left(z^\T A y\right) \rho_t(z) \d z \right)^2} + 2\right) \rho_t(y) \d y=0, & \Phi_0 = 0.
        \end{cases}
    \end{equation}
\end{proposition}

\subsection{Accelerated Softmax Attention Dynamics}

For $t > 0$, we approximate the density $\rho_t$ by the empirical measure 
$\frac{1}{N} \sum_{j = 1}^{N} \delta_{X_i(t)}$.
In Appendix \ref{sec:proofs}, we prove the following discretization result.

\begin{proposition}\label{lemma:acc_softmax_particles}
    The deterministic interacting particle system associated to \eqref{eq:acc_softmax} is for $i=1,\ldots,N$ given by
    \begin{equation*} \label{eq:acc_softmax_particles}
        \begin{cases}
            \dot{X}_i(t)
            = \displaystyle\frac{N B Y_i(t)}{\sum_{j = 1}^{N} M_{i, j}(t)}, \\ 
            \dot{Y}_i(t)
            = - \alpha_t Y_i(t) + \displaystyle\frac{N}{2} A \sum_{j = 1}^{N} X_j(t) M_{i, j}(t) \left( \displaystyle\frac{\| Y_i(t) \|_B^2  }{\left( \sum_{k = 1}^{N} M_{i, k}(t)\right)^2} + \frac{\| Y_j(t) \|_B^2}{\left(\sum_{k = 1}^{N} M_{j, k}(t) \right)^2} + 2\right),
        \end{cases}
    \end{equation*}
    where $Y_i(t) \coloneqq \nabla \Phi_t(X_i(t))$ and $M_{i, j}(t) \coloneqq \exp\left( X_i(t)^{\T} A X_j(t)\right)$. 
    \end{proposition}   
    
    \begin{remark}    Let $\mathbf{X}(t) \in \R^{N \times d}$ have the rows $X_i(t) \in \R^d$, and $\mathbf{Y}(t)$ be defined analogously. 
    Further, set $\mathbf{M}(t) \coloneqq \left(M_{i,j}(t) \right)_{i,j=1}^N$,
    $\mathbf{R} \coloneqq \diag\left( (\mathbf{Y} \odot \mathbf{Y} B^{\T}) \1_d \oslash \big((\mathbf{M} \1) \odot (\mathbf{M} \1)\big)\right)$, and let $\{ \mathbf{R}, \mathbf{M} \} \coloneqq \mathbf{R} \mathbf{M} + \mathbf{M} \mathbf{R}$ denote the anti-commutator. Further, denote elementwise multiplication and division by $\odot$ and $\oslash$, respectively.
    Then, we  can write the above system compactly as follows:
    \begin{equation} \label{eq:acc_softmax_matrices}
        \begin{cases}
            \dot{\mathbf{X}}(t)
            = N \diag( \mathbf{M}(t) \1)^{-1} \mathbf{Y}(t) B^{\T}, \\
            \dot{\mathbf{Y}}(t)
            = - \alpha_t \mathbf{Y}(t)
            + \frac{N}{2} \big( \{\mathbf{R}(t), \mathbf{M}(t) \} + 2 \mathbf{M}(t) \big) \mathbf{X}(t) A.
        \end{cases}
    \end{equation}   
        Consider the conservative part of \eqref{eq:acc_softmax_matrices}, i.e., \eqref{eq:acc_softmax_matrices} without the term $\alpha(t) \mathbf{Y}(t)$.
A Hamiltonian for the system is
\begin{equation} \label{eq:conservative_Hamiltonian_softmax}
  H^{\SM} \colon \R^{N \times d} \times \R^{N \times d} \to \R, 
   \quad (\mathbf{X}, \mathbf{Y})
    \mapsto \frac{N}{2} \tr(\diag(\mathbf{M} \1)^{-1} \mathbf{Y} B^{\T} \mathbf{Y}^{\T}) - \frac{N}{2} \1^{\T} \mathbf{M} \1.
\end{equation}
The second summand of $H^{\SM}$ is a constant multiple of the particle approximation of the potential energy \eqref{eq:softmaX_interaction_energy}.
This Hamiltonian is non-separable: its first summand depends on $\mathbf{Y}$ \textit{and on} $\mathbf{X}$ simultaneously, resulting in a major difference between our work and the scheme from \cite{ZPR2026}.
\end{remark}

We now apply the time discretizations outlined in \Cref{subsec:time_discretizations} to the accelerated transformer with softmax self-attention \eqref{eq:acc_softmax_matrices}.
We view the discretized time update as the accelerated softmax attention layer.

\begin{example}
    The time-dependent Hamiltonian belonging to the conservative softmax-Hamiltonian \eqref{eq:conservative_Hamiltonian_softmax} for \eqref{eq:acc_softmax_matrices} is
    \begin{equation} \label{eq:softmax_Hamiltonian}
    \begin{aligned}
        \tilde{H}^{\SM} \colon \R \times \R^{N \times d} \times \R^{N \times d} & \to \R, \\
        (t, \mathbf{Q}, \mathbf{P}) & \mapsto \frac{N}{2} \left(e^{-\eta(t)} \tr\left( \diag(M(\mathbf{Q}) \1)^{-1} \mathbf{P} B^{\T} \mathbf{P}^{\T}\right) - e^{\eta(t)} \1^{\T} M(\mathbf{Q}) \1 \right).
    \end{aligned}
    \end{equation}
\end{example}

\begin{example}
For softmax self-attention, the plain explicit Euler method \eqref{eq:plain_explicit_Euler} becomes
\begin{equation*} \label{eq:softmax_plain_explicit_Euler}
    \begin{cases}
    \mathbf{M}^{(k)}
    = \exp\left( \mathbf{X}^{(k)} A (\mathbf{X}^{(k)})^{\T}\right), \\
    \mathbf{X}^{(k + 1)}
    = \mathbf{X}^{(k)}
    + \tau N \diag(\mathbf{M}^{(k)} \1)^{-1} \mathbf{Y}^{(k)} B^{\T}, \\
    \mathbf{R}^{(k)}
    = \diag\big(\mathbf{Y}^{(k)} \odot \mathbf{Y}^{(k)} B) \1_d \oslash \big((\mathbf{M}^{(k 
    )} \1) \odot (\mathbf{M}^{(k 
    )} \1)\big)\big), \\
    \mathbf{Y}^{(k + 1)}
    = - \alpha_k \mathbf{Y}^{(k)} + \frac{\tau}{2} N \big( \{ \mathbf{R}^{(k)}, \mathbf{M}^{(k 
    )} \} + 2 \mathbf{M}^{(k 
    )}\big) \mathbf{X}^{(k + 1)} A.
    \end{cases}
\end{equation*}
Further, the \textit{exponential Euler method} \eqref{eq:Presympl_exp_euler} satisfies 
    \begin{equation*}
        \begin{cases}
            \mathbf{M}^{(k)}
            = \exp\left( \mathbf{X}^{(k)} A (\mathbf{X}^{(k)})^{\T}\right), \\
            \mathbf{X}^{(k + 1)}
            = \mathbf{X}^{(k)} + h N \diag(\mathbf{M}^{(k)} \1)^{-1} \mathbf{Y}^{(k)} B^{\tT}, \\
            \mathbf{R}^{(k)}
            = \diag\big(\mathbf{Y}^{(k)} \odot \mathbf{Y}^{(k)} B) \1_d \oslash \big((\mathbf{M}^{(k 
            )} \1) \odot (\mathbf{M}^{(k 
            )} \1)\big)\big), \\
            \mathbf{Y}^{(k + 1)} = \sigma_k \mathbf{Y}^{(k)} + h \int_{0}^{h} \exp\left( - \int_{s}^{h} \alpha(z) \d{z}\right) \d{s} \frac{N}{2} \left( \{ \mathbf{R}^{(k)}, \mathbf{M}^{(k)} \} + 2 \mathbf{M}^{(k)} \right) \mathbf{X}^{(k)} A.
        \end{cases}
    \end{equation*}
\end{example}

\section{Numerical Results} \label{sec:numerics}
Finally, we demonstrate the numerical performance of our method\footnote{The code reproducing the experiments is available on GitHub: \url{https://github.com/ViktorAJStein/SympFormer}.}.
The goal of the decoder-only causal nanoGTP-style transformer is to predict the next token.
All models end with a final LayerNorm and a linear head, and each layer learns the matrices $Q$, $K$, and $V$.
While the theory above models only single-head self-attention, in our implementation of the accelerated softmax-attention, the score matrices are averaged across attention heads before exponentiation.
To evaluate performance, we use a cross-entropy loss.

We observe that the oracle-preserving SympFormer outperforms the baseline and the Yuriiformer architecture, achieving lower validation loss at higher wall-time costs for the same number of iterations.

\subsection{The SympFormer architecture}
In the following, we explain the differences between the mathematical formalism introduced above and the implementation.
The precise implementation details are deferred to \Cref{section:Implementation}.

\paragraph{Simplified time discretization}
Let us denote the right-hand sides of the two update equations in \eqref{eq:acc_softmax_matrices} resp. \eqref{eq:acc_linear_matrix} by $F(\mathbf{X}, \mathbf{Y})$ and $G(\mathbf{X}, \mathbf{Y})$, respectively, ignoring the damping term $\alpha(t) \mathbf{Y}(t)$.
In terms of compute, the bottleneck of the transformer architecture is evaluating the self-attention kernel $\kappa$ from \eqref{eq:self_attn_kernel}.
To limit the number of computations of $\kappa$, the SympFormer uses a \textit{simultaneous} forward Euler update on both $\mathbf{X}$ and $\mathbf{Y}$ instead of the conformally symplectic Euler scheme~\eqref{eq:presymp_Euler}, in which the momentum is updated first, and then the position update uses the \textit{new} momentum. 

To construct a learnable but principled damping schedule, we choose the log-linear damping schedule $\alpha(t) = \frac{r}{t} + m$ \cite{FJV2021} for per-layer learnable parameters $r, m > 0$.
Heuristically, this choice of damping schedule ensures that the damping decreases over time, as in the usual Nesterov scheme, and also sets a \enquote{damping floor} $m$ to prevent the damping from becoming too small at later times.
Intuitively, for large values of $t > 0$, the particles should be in a region where the objectives $\F^{\SM}$ resp. $\F^{\lin}$ is strongly convex, so an approximately constant damping is the correct schedule.

We also learn separate step sizes $h_{k, X}$ and $h_{k, Y}$ for the position and momentum updates, respectively, giving the optimizer independent control over both.


In the implementation, each layer carries its own attention-step parameters and its own damping schedule.
Hence, the MLP substep reuses the attention momentum stream as its look-ahead velocity.

We summarize the architecture in \Cref{algo:SympFormer} and visualize it in \figref{fig:attn-MLP}.

\begin{algorithm}[t]
  \caption{Forward pass through the SympFormer -- an accelerated transformer.}
  \label{algo:SympFormer}
  \DontPrintSemicolon
  \KwIn{Batch size $B \in \N$, token indices $\mathrm{idx} \in \{0,\dots,V{-}1\}^{B \times N}$, initial time $t_0$ \;
        $L$ blocks, each with learned scalar parameters
        $\hX_\ell, \hY_\ell > 0$ (position/momentum step sizes),
        MLP scalars $m_{\MLP,\ell}, \beta_{\MLP,\ell} \in (0,1)$, $\gamma_{\MLP,\ell} > 0$\;
        learned damping coefficients
        $c_{\log}, c_{\mathrm{lin}} > 0$
        with $\eta(t) = c_{\log}\ln t + c_{\mathrm{lin}}\,t$, which are all used to build the coefficients $\zeta_{1, \ell}$, $\zeta_{2, \ell}$, defined in \cref{tab:zeta}, depending on the time discretization.}
  \KwOut{Logits $\in \R^{B \times N \times V}$.}

  $x \gets \mathrm{TokEmb}(\mathrm{idx}) + \mathrm{PosEmb}(\mathrm{pos})$
  \tcp*[r]{token features}
  $y \gets 0$ 
  \tcp*[r]{attention momentum}

  \For{$\ell = 0, 1, \ldots, L-1$}{
    evaluate oracles $F$ and $G$ (right hand side in \eqref{eq:acc_softmax_matrices} or \eqref{eq:acc_linear_matrix}) 
    \tcp*[r]{Accelerated attention step}
    \;
    $\alpha_\ell \gets c_{\log}/t_\ell + c_{\mathrm{lin}}$
    \tcp*[r]{instantaneous damping rate $\dot\eta(t_\ell)$}\;
    $y \;\gets\; \zeta_{1, \ell} y + \hY_\ell \zeta_{2, \ell} G $
    \tcp*[r]{forward step on $y$}\;
    $x_{\ell + \frac{1}{2}} \;\gets\; x + \hX_\ell\, F$
    \tcp*[r]{forward step on $x$}\;
    $y_{\ell + \frac{1}{2}} \;\gets\; \LN_Y(y)$
    \tcp*[r]{momentum LayerNorm}\;
    \tcp*[r]{Accelerated MLP step}
    $\tilde x_\ell \gets x_{\ell+\frac12} + m_{\MLP,\ell}\,y_{\ell+\frac12}$\;
    $d_\ell \gets \MLP_\ell(\LN(\tilde x_\ell))$\;
    $y \gets \LN_V\!\bigl(\beta_{\MLP,\ell}\,y_{\ell+\frac12} + \gamma_{\MLP,\ell}\, d_\ell\bigr)$\;
    $x \gets x_{\ell+\frac12} + y$\;
    $t \gets t + \hX_\ell$\;
  }

  $\mathrm{logits} \gets \LN_f(x)\,W_{\mathrm{lm}}^\top$
  \tcp*[r]{final layer norm and linear head}

  \Return{$\mathrm{logits}$}
\end{algorithm}

\begin{figure}[t]
    \centering
        \resizebox{\textwidth}{!}{%
\begin{tikzpicture}[
  x=1cm,y=1cm,>=stealth,thick,
  every node/.style={font=\small},
  state/.style={draw,circle,minimum size=8mm,inner sep=0pt},
  plus/.style={draw,circle,minimum size=5.2mm,inner sep=0pt},
  block/.style={draw,minimum width=13mm,minimum height=8mm,inner sep=0pt},
  lnv/.style={draw,minimum width=11mm,minimum height=7mm,inner sep=1pt},
  lab/.style={font=\scriptsize,inner sep=1pt},
  look/.style={font=\scriptsize}
]

\node[state] (x0) at (-1,1.1) {$x_\ell$};
\node[state] (y0) at (-1,-.5) {$y_\ell$};
\node[state] (xin1) at (.5,.25) {$\bar x_\ell$};
\draw[->] (xin1) -- (1.5,.25);
\draw[->] (x0) to node[pos=0.55,above,lab] {$~~~\LN$} (xin1);
\draw[->] (y0) to node[pos=0.63,below,lab] {$\mu^{\mathrm{att}}_\ell$} (xin1);
\node[block] (attn) at (2.2,0.25) {$\Attn_\ell$};
\node[plus]  (px1)  at (4.45,1.1) {$+$};
\node[plus]  (py1)  at (4.45,-.5) {$+$};
\node[lnv]   (ln1)  at (5.75,-.5) {$\LN_Y$};
\node[state] (xh)   at (7.1,1.1) {$x_{\ell+\frac12}$};
\node[state] (yh)   at (7.1,-.5) {$y_{\ell+\frac12}$};
\draw[->] (x0) -- (px1);
\draw[->] (y0) -- node[below,lab] {$\zeta_{1,\ell}$} (py1);
\draw[->] (attn.east) to node[pos=0.65,above,lab] {$\hX_{\ell}F_\ell\qquad$} (px1);
\draw[->] (attn.east) to node[pos=0.8,above] {$\hY_{\ell}\zeta_{2,\ell}G_\ell$} (py1);
\draw[->] (px1) -- (xh);
\draw[->] (py1) -- (ln1);
\draw[->] (ln1) -- (yh);
\node[state] (xin2) at (9.35,.25) {$\tilde x_\ell$};
\draw[->] (xh) to node[pos=0.56,above,lab] {$~~~\LN$} (xin2);
\draw[->] (yh) to node[above=2mm,lab] {$m_{\MLP,\ell}$} (xin2);
\node[block] (mlp) at (10.75,.25) {$\MLP_\ell$};
\draw[->] (xin2) -- (mlp.west);
\node[plus]  (py2) at (11.95,-.5) {$+$};
\node[lnv]   (ln2) at (13.25,-.5) {$\LN_Y$};
\node[state] (y1)  at (14.8,-.5) {$y_{\ell+1}$};
\node[plus]  (px2) at (14.8,1.1) {$+$};
\node[state] (x1)  at (15.65,1.1) {$x_{\ell+1}$};
\draw[->] (yh) -- node[below,lab] {$\beta_{\MLP,\ell}$} (py2);
\draw[->] (mlp.east) -- node[right,lab] {$\gamma_{\MLP,\ell} \cdot d_\ell$} (py2);
\draw[->] (py2) -- (ln2);
\draw[->] (ln2) -- (y1);
\draw[->] (xh) -- (px2);
\draw[->] (y1) -- (px2);
\draw[->] (px2) -- (x1);
\end{tikzpicture}}
    \caption{The accelerated attention block followed by an accelerated MLP block as described in \Cref{algo:SympFormer}.}
    \label{fig:attn-MLP}
\end{figure}
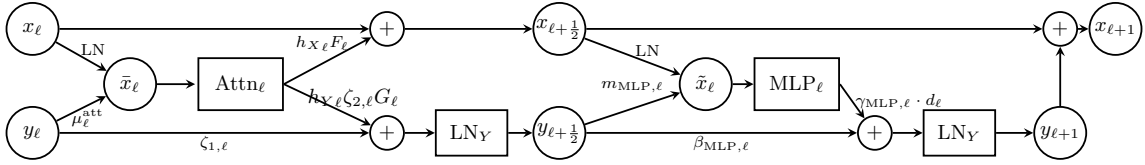

\subsection{Comparisons on the TinyStories data set} 
Both for softmax attention and for linear attention, we compare the vanilla transformer architecture (\enquote{Baseline}) and the Lie-Trotter Nesterov variant of the Yuriiformer architecture \cite{ZPR2026} with a plain Euler forward discretization (\enquote{Plain Euler}), the presymplectic Euler method (\enquote{Presymp Euler}), the presymplectic exponential Euler (\enquote{Presymp ExpEuler}), and the AB-2 scheme \eqref{eq:AB2} applied to \eqref{hamil} (\enquote{Presymp AB2}) and to \eqref{eq:time_dependent_Hamiltonian_system} (\enquote{Presymp ETD-AB2}).

We report numerical results on the \href{https://www.kaggle.com/datasets/thedevastator/tinystories-narrative-classification}{TinyStories} data set, which contains the text from over 2000 short stories in \Cref{tab:tinystories_softmax_1}. 
The exact configurations (number of layers, number of heads, embedding dimension, etc.) are listed in \Cref{section:Implementation}.

\begin{table}[H]
    \centering
    \begin{tabular}{c|c|c|c}
        \textbf{model}       & \textbf{last} ($\downarrow$)      & \textbf{best} ($\downarrow$)     & \textbf{wall clock time} (seconds) ($\downarrow$) \\ \hline
        Baseline    & 2.4687                   & 2.4473                & \textbf{2409.4} \\
        YuriiFormer-Lie-Trotter    & 2.4041                   & 2.3872                & 2835.2 \\
        Plain Euler    & 2.3247                   & 2.3234                & 5036.1 \\
        Presymp Euler    & 2.4728                   & 2.4592                & 4979.0 \\
        Presymp ExpEuler    & 2.3579                   & 2.2523                & 5623.4 \\
        Presymp AB2    & 2.6653                   & 2.6546                & 5249.4 \\
        Presymp ETD-AB2    & \textbf{1.8386}                   & \textbf{1.8386}                & 5798.2 \\
    \end{tabular}
    \caption{Validation loss on the tinystories data set after 10000 optimization steps.}
    \label{tab:tinystories_softmax_1}
\end{table}

\begin{figure}[H]
    \centering
    \includegraphics[width=0.85\linewidth]{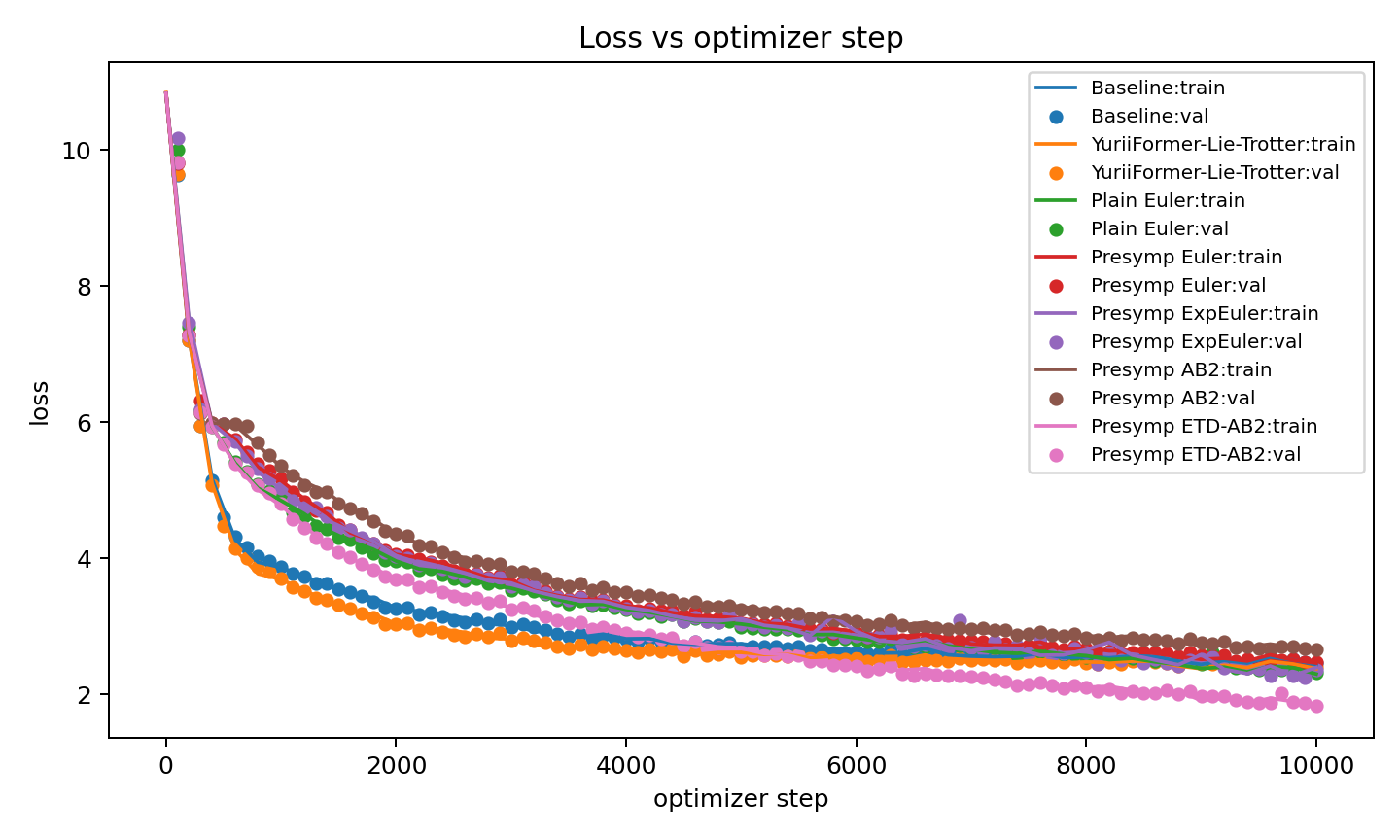}
    \caption{Validation loss (circles) and training loss (lines) on the tinystories data set after 10000 optimization steps, illustrating the results from \Cref{tab:tinystories_softmax_1}.}
    \label{fig:placeholder}
\end{figure}

Lastly, we also report results for the SympFormer with linear attention in \Cref{tab:linear_tinystories1,tab:linear_tinystories2}.
Again, the exact configurations (number of layers, number of heads, embedding dimension, etc.) are listed in \Cref{section:Implementation}.
\begin{table}[H]
    \centering
    \begin{tabular}{c|c|c|c}
        model                   & last ($\downarrow$)      & best ($\downarrow$)     & wall clock time (seconds) ($\downarrow$) \\ \hline
        Baseline                & 3.0830                   & 2.961427                & \textbf{486.7} \\
        YuriiFormer & 2.8783                   & 2.765975                & 749.6 \\
        plain Euler             & 3.0399                   & 2.920602                & 809.3 \\
        Presymp Euler           & \textbf{2.8467}          & \textbf{2.731363}       & 1172.8 \\
    \end{tabular}
    \caption{Validation loss of the SympFormer with linear attention (smaller configuration) on the TinyStories data set after 10000 optimization steps.}
    \label{tab:linear_tinystories1}
\end{table}

\begin{table}[H]
    \centering
    \begin{tabular}{c|c|c|c}
        model                   & last ($\downarrow$)      & best ($\downarrow$)     & wall clock time (seconds) ($\downarrow$) \\ \hline
        Baseline                & 3.0118                   & 2.917168                & \textbf{8956.8} \\
        YuriiFormer & 2.8486                   & 2.765023                & 9601.3 \\
        plain Euler             & \textbf{2.7483}          & \textbf{2.656692}        & 10806.5 \\
        Presymp Euler           & 2.7862                   & 2.690885                & 12620.0 \\
    \end{tabular}
    \caption{Validation loss of the SympFormer with linear attention (larger configuration) on the TinyStories data set after 10000 optimization steps.}
    \label{tab:linear_tinystories2}
\end{table}

\section{Conclusion and Outlook}

In this work, we introduce a new transformer architecture inspired by accelerated optimization algorithms on density manifolds.
Viewing the attention block update as a time and space discretization of a gradient flow in probability density space enabled us to construct the associated \enquote{accelerated} attention PDEs with a \enquote{self-attention-induced Hamiltonian functional}.
After approximating the PDE using particles, we outlined different time discretizations that respect the underlying damped Hamiltonian dynamics.
Our numerical experiments showed that our novel SympFormer architecture outperforms the vanilla transformer architecture on small datasets while requiring the same number of oracle calls.

In future work, we shall explore other forms of self-attention layers by formulating the associated transformer PDE as a gradient flow in a suitable generalized Wasserstein geometry and then applying analogous acceleration procedures.
We intend to put this framework on a rigorous theoretical footing and understand the convergence properties of the proposed algorithms.
Lastly, larger-scale experiments of accelerated attention dynamics and the evaluation on downstream tasks are still needed to further cement the utility of our method.
\\[5mm]

\noindent\textbf{Acknowledgments.}  The paper was started during a visit of W. Li to the
TU Berlin supported by the SPP 2298 \enquote{Theoretical Foundations of Deep Learning} in Summer 2025. 
W. Li's work is supported by the AFOSR YIP award No. FA9550-23-1-0087, NSF RTG: 2038080, NSF DMS-2245097, and the McCausland Faculty Fellow in University of South Carolina. GS acknowledges funding from the DFG Cluster of Excellence MATH+ and from the CNRS as a CNRS fellow ambassador 2026.

\begin{center}
  \FundingLogos
  
  \vspace{0.5em}
  \begin{tcolorbox}\centering\small   
    V.S. acknowledges funding from the European Research Council (ERC) under the European Union’s Horizon Europe research and innovation programme (grant agreement No. 101198055, project acronym NEITALG).
    
    Views and opinions expressed are, however, those of the author(s) only and do not necessarily reflect those of the European Union or the European Research Council Executive Agency. Neither the European Union nor any of the aforementioned granting authorities can be held responsible for them.
  \end{tcolorbox}
\end{center}

\bibliographystyle{abbrv}
\bibliography{Bibliography}

\clearpage
\appendix

\section{Proofs} \label{sec:proofs}

\textbf{Proof of Proposition \ref{lemma:acc_linear}}
It remains to compute 
$\frac{\delta}{\delta \rho} \frac12 \int \Phi (G_\rho^{\lin})^{-1}[\Phi] \d x = \frac{\delta}{\delta \rho} \Psi(\rho)$, where
    \begin{equation*}
        \Psi(\rho) \coloneqq
       \frac{1}{2} \int_{} \Phi(x) \left( - \nabla \cdot \left( \rho(x) \int_{} y^{\T} A x \nabla \Phi(y) \rho(y) \d{y} \right)\right) \d{x}.
    \end{equation*}
    Applying integration by parts, we see that
    \begin{equation*}
        \Psi(\rho)
        = \frac{1}{2} \int_{} \int_{} \langle \nabla \Phi(x), \nabla \Phi(y) \rangle y^{\T} A x \rho(x) \rho(y) \d x \d y.
    \end{equation*}
    This is an interaction energy, whose functional derivative is known to be
    \begin{equation*}
        \Big[\frac{\delta \Psi}{\delta \rho} (\rho) \Big](x)
        = \int_{} \langle \nabla \Phi(x), \nabla \Phi(y) \rangle y^{\T} A x  \, \rho(y)  \d y.\qquad \Box
    \end{equation*}
\hspace{0.5cm}

\noindent
\textbf{Proof of Proposition \ref{lemma:preservation_elliptic}.}
    Let  $\rho_t$,  $t \ge 0$ be produced by \eqref{eq:acc_lin_transformer_PDE}, 
    where $\rho_0 = E(m_0, \Sigma_0, g)$.
    \\[1ex]
    1.    First, we show that the ellipticity is preserved along the flow.  
        The velocity field for the evolution of $\rho_t$ given by
        \begin{equation} \label{eq:velocity_field_u}
            u_t \colon \R^d \to \R^d, \qquad
            x \mapsto \int_{} \nabla \Phi_t(y) y^{\T} \rho_t(y) \d{y} \, A x
            \eqqcolon C_t x
        \end{equation}
        is a linear map.
        By \cite[Lemma~8.1.6]{AGS2008} 
        we have $\rho_t = (X_t)_{\#} \rho_0$ for all $t \ge 0$, where $X_t \colon \R^d \to \R^d$ solves
        \begin{equation} \label{eq:M_evolution}
            \begin{cases}
                \dot{X}_t(x) = C_t X_t(x), & t > 0, \\
                X_0(x) = x.
            \end{cases}        
        \end{equation}
        The solution of this linear ODE is given by $X_t(x) = G_t x$ with
        \begin{equation*}
            \begin{cases}
                \dot{G}_t = C_t G_t, & t > 0, \\
                G_0 = I_d.
            \end{cases}
        \end{equation*}
        Then, $\rho_t = E(G_t m_0, G_t \Sigma_0 G_t^{\T}, g)$. Further,  $m_t = G_t m_0$ and $\Sigma_t = G_t \Sigma_0 G_t^{\T}$ fulfill \eqref{determ}.        
\\[1ex]
2.
        Now, we prove the expression for $\Phi_t$.
        In our notation, the second equation of \eqref{eq:acc_lin_transformer_PDE} becomes
        \begin{equation} \label{eq:Phi_t_reduced}
            \partial_t \Phi_t(x)
            + x^{\T} C_t^{\T} \nabla \Phi_t(x) 
            = - \alpha_t \Phi_t(x) + \frac{1}{2} x^{\T} V x.
        \end{equation}
        Since $\Phi_0 = 0$, the function $\Phi_t$ will be a quadratic function for all $t > 0$ and we can make the ansatz $\Phi_t(x) = \frac{1}{2} x^{\T} P_t x + b_t^{\T} x + c_t$.
        Since $P_0 = 0$ and $V$ is symmetric, $P_t$ is symmetric for all $t > 0$.
        Hence, $\partial_t \Phi_t(x) = \frac{1}{2} x^{\T} \dot{P}_t x + \dot{b}_t^{\T} x + \dot{c}_t$ and $\nabla \Phi_t(x) = P_t x + b_t$, so that \eqref{eq:Phi_t_reduced} becomes
        \begin{equation*}
            \frac{1}{2} x^{\T} \dot{P}_t x + \dot{b}_t^{\T} x + \dot{c}_t
            + x^{\T} C_t^{\T} P_t x + x^{\T} C_t^{\T} b_t
            = - \alpha_t \left( \frac{1}{2} x^{\T} P_t x + b_t^{\T} x + c_t \right)
            + \frac{1}{2} x^{\T} V x.
        \end{equation*}
        Comparing coefficients yields
        \begin{equation*}
            \begin{cases}
                \dot{P}_t + C_t^{\T} P_t + P_t C_t
                = - \alpha_t P_t + V, & P_0 = 0, \\
                \dot{b}_t + C_t^{\T} b_t
                = - \alpha_t b_t, & b_0 = 0, \\
                \dot{c}_t
                = - \alpha_t c_t, & c_0 = 0.
            \end{cases}
        \end{equation*}
        Thus, $b_t \equiv 0$ and $c_t \equiv 0$, so $\Phi_t$ remains a \enquote{purely quadratic} function: $\Phi_t(x) = \frac{1}{2} x^{\T} P_t x$.
\\[1ex]
3.         Now, we can derive the system \eqref{eq:C_t_and_P_t}.
        Given the closed-form expression for $\Phi_t$, we can simplify the velocity field \eqref{eq:velocity_field_u} as
        \begin{equation*}
            C_t
            = \int_{\R^d} \nabla \Phi_t(y) y^{\T} \rho_t(y) \d{y} \, A
            = P_t \int_{\R^d} y y^{\T} \rho_t(y) \d{y} \, A
            = P_t (\kappa_g \Sigma_t + m_t m_t^{\T}) A.
        \qquad \Box
        \end{equation*}         
\hspace{0.5cm}

\noindent
\textbf{Proof of Proposition \ref{prop:add}}
    The first equation follows exactly as in the proof of Proposition \ref{lemma:acc_softmax_particles}.
    We have
    \begin{align*}
        \partial_t \nabla \Phi_t(x)
        = - \alpha_t \nabla \Phi_t(x) - \int_{\R^d} \left(\langle \nabla \Phi_t(x), \nabla \Phi_t(y) \rangle A y + \nabla^2 \Phi_t(x) \nabla \Phi_t(y) x^{\T} A y\right) \rho_t(y) \d y + V x.
    \end{align*}
    Plugging in $x = X_i(t)$ and $\rho_t \leftrightarrow \frac{1}{N} \sum_{j = 1}^{N} \delta_{X_j(t)}$ yields
    \begin{align*}
        \dot{Y}_i(t)
        & = (\partial_t \nabla \Phi_t)(X_i(t)) + \nabla^2 \Phi_t(X_i(t)) \dot{X}_i(t) \\
        & = - \alpha_t Y_i(t) - \frac{1}{N} \sum_{j = 1}^{N} \langle Y_i(t), Y_j(t) \rangle A X_j(t) - \frac{1}{N} \nabla^2 \Phi_t(X_i(t)) \sum_{j = 1}^{N} Y_j(t) X_i(t)^{\T} A X_j(t) \\
        & \qquad + V X_i(t)
        + \nabla^2 \Phi_t(X_i(t)) \frac{1}{N} \sum_{j = 1}^{N} Y_j(t) X_j(t)^{\T} A X_i(t) \\
        & = - \alpha_t Y_i(t) - \frac{1}{N} \sum_{j = 1}^{N} \langle Y_i(t), Y_j(t) \rangle A X_j(t) + V X_i(t).\qquad \Box
    \end{align*}
    \vspace{0.2cm}
    
 \noindent
\textbf{Proof of Proposition \ref{lemma:acc_softmax}}
The first equation follows directly from \eqref{eq}.
\\
Let $\varphi \in C^\infty(\Omega)$ with mass zero such that 
$\rho + t \varphi \in \mathcal P(\Omega)$ for all sufficiently small $t \in \R$.
From the integration by parts, we have 
        \begin{align*}
             &\frac{\d}{\d t} \bigg|_{t = 0} \int_{} \Phi(x) (G^{\text{SM}}_{\rho + t \varphi})^{-1}[\Phi](x) \d x\\
                        &=\int_{} \nabla \Phi(x) \cdot \left( 
            \frac{\d}{\d t} \bigg|_{t = 0} 
            \frac{\rho(x) + t \varphi(x)}{\int_{} \exp\left(z^\T A x\right) (\rho(z) + t \varphi(z)) \d z} V A^{-1} \nabla \Phi(x) \right) \d x 
            \\
            & = \int_{} \nabla \Phi(x) \cdot \bigg( \bigg( \frac{\varphi(x)}{\int \exp\left( z^\T A x\right) \rho(z) \d z}  - \rho(x) \frac{\int \exp\left(z^\T A x\right) \varphi(z) \d z}{\left(\int \exp\left(z^\T A x\right) \rho(z) \d z\right)^2} \bigg) V A^{-1} \nabla \Phi(x) \bigg) \d{x} 
            \\
            & = \int_{} 
            \frac{\langle \nabla\Phi(x), V A^{-1} \nabla \Phi(x) \rangle}{\int_{} \exp\left(z^\T A x\right) \rho(z) \d z} \varphi(x) \d x  \\
            & \quad - \int_{} \int 
            \frac{\exp\left(y^\T A x\right)}{\left(\int \exp\left(z^\T A x\right) \rho(z) \d z\right)^2} 
            \langle \nabla \Phi(x), 
            V A^{-1} \nabla \Phi(x) \rangle \rho(x) \varphi(y) \d x \d y, 
        \end{align*} 
        so that
        \begin{align*}
            & \frac{\delta}{\delta \rho}\left\{\int \Phi(x) (G_{\rho + t \phi}^{\SM})^{-1}[\Phi](x) \d{x}\right\}(x) \\
            & \qquad = \frac{\| \nabla \Phi(x) \|_B^2}{\int \exp\left(z^{\T} A x\right) \rho(z) \d{z}} 
            - \int \frac{\exp\left(x^{\T} A y\right)}{\left( \int \exp\left(z^{\T} A x\right) \rho(z) \d z\right)^2} \| \nabla \Phi(y) \|_B^2 \rho(y) \d{y}.
        \end{align*}
Substituting both expressions together with \eqref{helper_a} into \eqref{sys_1} yields the assertion. \hfill $\Box$        
\medskip

\noindent
\textbf{Proof of Proposition \ref{lemma:acc_softmax_particles}}
Using $Y_i(t) \coloneqq \nabla \Phi_t(X_i(t))$, 
the first equation in \eqref{eq:acc_softmax} becomes
\begin{equation*}
    \dot{X}_i(t)
    = \frac{V A^{-1} Y_i(t)}{\frac{1}{N} \sum_{j = 1}^{N} \exp\left( X_j(t)^{\T} A X_i(t)\right)},
\end{equation*}
see also \cite[Lemma~8.1.6]{AGS2008}.
For the second equation, we use a similar \textit{ansatz} as in \cite[App.~C.1]{SL2026,SL2026b}. First, we compute
\begin{equation} \label{eq:dot_P}
    \dot{Y}_i(t)
    = (\partial_t \nabla \Phi_t)( X_i(t) ) + \nabla^2 \Phi_t(X_i(t)) \dot{X}_i(t).
\end{equation}
The first term on the right-hand side can be rewritten as follows:
Taking the spatial gradient in the second equation of \eqref{eq:acc_softmax} yields (skipping the index for simplicity)
\begin{align*}
    (\partial_t \nabla \Phi_t)(X_t)
    & = - \alpha_t \nabla \Phi_t(X_t)
    - \nabla^2 \Phi_t(X_t) \underbrace{\frac{B \nabla \Phi_t(X_t)}{\int_{} \exp\left( y^\T A X_t \right) \rho_t(y) \d y}}_{= \dot{X}_t} \\
    & \quad + \frac{1}{2} \frac{1}{\left( \int_{} \exp\left(y ^\T A X_t \right) \rho_t(y) \d{y} \right)^2} \| \nabla \Phi_t(X_t) \|_B^2 \int_{} A y \exp\left( y ^\T A X_t\right) \rho_t(y) \d{y} \\
    & \quad + 
    \frac{1}{2} \int_{} \exp\left( y^{\T} A X_t\right) \left( \frac{\| \nabla \Phi_t(y) \|_B^2}{\left( \int_{} \exp\left( z^\T A y\right) \rho_t(z) \d{z} \right)^2} + 2\right) \, A y \, \rho_t(y) \d{y}.
\end{align*}
Plugging this into \eqref{eq:dot_P}, we see that the terms involving $\nabla^2 \Phi_t$ cancel, so we obtain
\begin{align*}
    \dot{Y}_i(t)
    & = - \alpha_t Y_i(t) 
    + \frac{1}{2} \frac{1}{\left( \int_{} \exp\left( y^\T A X_t \right) \rho_t(y) \d{y} \right)^2} \| Y_i(t) \|_B^2 \int_{} A y \exp\left( y^\T A X_t\right) \rho_t(y) \d{y} \\
    & \quad + \frac{1}{2} \int_{} \exp\left( y^\T A X_t  \right) \left( \frac{\| \nabla \Phi_t(y) \|_B^2}{\left( \int_{} \exp\left( z^\T A y \right) \rho_t(z) \d{z} \right)^2} + 2\right) A y \, \rho_t(y) \d{y} \\
    & = - \alpha_t Y_i(t)
    + \frac{N}{2} A \sum_{j = 1}^{N} X_j(t) M_{i, j}(t) \left( \frac{\| Y_i(t) \|_B^2  }{\left( \sum_{j = 1}^{N} M_{i, j}(t)\right)^2} + \frac{\| Y_j(t) \|_B^2}{\left(\sum_{k = 1}^{N} M_{j, k}(t) \right)^2} + 2\right),
\end{align*}
where $M_{i, j}(t) \coloneqq \exp\left( X_j(t)^{\T} A X_i(t)\right)$. \hfill $\Box$

\section{Implementation details} \label{section:Implementation}

The MLP uses the GELU activation function with 4× expansion, zero dropout, no bias terms, and a global grad-norm clip at 1.0.
We train in float16 and use a cosine LR schedule with linear warm-up, with the minimum LR set to 0.1×peak.

The scalar parameters are jointly optimized with the MLP weights using backpropagation and are constrained to the appropriate domain via a softplus or sigmoid parametrization.
They form a separate parameter group for the optimizer, are excluded from gradient clipping, and receive a 5x learning rate boost.

We use the GPT-2 BPE tokenizer and pad the vocabulary to V=50304.
Unlike \cite{ZPR2026}, we use only the AdamW optimizer rather than a mixed AdamW-Muon optimizer.

Next, we list the scalar prefactors $\zeta_1, \zeta_2 > 0$ used in \Cref{algo:SympFormer} by the first-order methods in \Cref{tab:zeta}.
\begin{table}[H]
    \centering
    \begin{tabular}{c|cc}
        architecture    & $\zeta_1$ & $\zeta_2$\\ \hline
        plain Euler     & $\alpha$  & $1$ \\
        Presymp Euler   & $1 - \alpha(t_k) h_Y $    & $1$ \\
        Presymp ExpEuler& $e^{-\Delta \eta_k}$      & $\frac{1 - e^{-\Delta \eta_k}}{\alpha_{\text{eff}}}$
    \end{tabular}
    \caption{Scalar prefactors $\zeta_i$ for the different first-order time discretizations.
    Here, $\alpha \in (0, 1)$ is a freely learned scalar.}
    \label{tab:zeta}
\end{table}

\paragraph{Hyperparameters for each experiment}
For \Cref{tab:tinystories_softmax_1}, we use 8 layers, 8 heads, embedding dimension 64, block size 512, batch size 6, and 16 gradient accumulation steps, evaluation batches of size 20, 150 warm-up steps, a minimum learning-rate ratio of 0.4, and a peak learning rate of $0.0003$.
The learned integrator step size is initialized with $h_0=0.1$.

For the experiments in \Cref{tab:linear_tinystories1}, we use 4 layers, 4 heads, an embedding dimension of 64, a block size of 128, a batch size of 6, 12 gradient accumulation steps, evaluation batches of size 20, 150 warm-up steps, a minimum learning rate ratio of 0.4, and a peak learning rate of 3e-4, and a scalar learning rate multiplier of 100. The time steps $h_k$ are learned and initialized with $h_0=0.25$ for the first-order methods resp. $h_0=0.01$ for the second-order methods.

For the experiments in \Cref{tab:linear_tinystories2}, we use 8 layers, 8 heads, an embedding dimension of 64, a block size of 1024, a batch size of 6, 12 gradient accumulation steps, evaluation batches of size 20, 150 warm-up steps, a minimum learning rate ratio of 0.4, and a peak learning rate of 3e-4.
The learned scalars have a 100 times learning rate multiplier, and we use a look-ahead step for the presymplectic discretizations.
The time steps $h_k$ are learned and initialized with $h_0=0.3$ for the first-order methods resp. $h_0=0.01$ for the second-order methods.

All experiments are implemented in PyTorch and run on NVIDIA Ada Generation RTX 6000 GPUs.

\end{document}